\def\year{2017}\relax
\documentclass[letterpaper]{article}
\usepackage{aaai17}
\usepackage{times}
\usepackage{helvet}
\usepackage{courier}
\usepackage{amsthm}
\usepackage{amsmath}
\usepackage{algorithm}
\usepackage{algorithmic}
\usepackage{graphicx}
\usepackage{multirow}
\usepackage{subfigure}

\nocopyright
\theoremstyle{plain}
\newtheorem{prop}{Proposition}

\begin{document}
\title{Multi-View Constraint Propagation with Consensus Prior Knowledge}
\author{}
\maketitle
\begin{abstract}
    In many applications, the pairwise constraint is a kind of weaker supervisory information which can be collected easily. The constraint propagation has been proved to be a success of exploiting such side-information. In recent years, some methods of multi-view constraint propagation have been proposed. However, the problem of reasonably fusing different views remains unaddressed. In this paper, we present a method dubbed Consensus Prior Constraint Propagation (CPCP), which can provide the prior knowledge of the robustness of each data instance and its neighborhood. With the robustness generated from the consensus information of each view, we build a unified affinity matrix as a result of the propagation. Specifically, we fuse the affinity of different views at a data instance level instead of a view level. This paper also introduces an approach to deal with the imbalance between the positive and negative constraints. The proposed method has been tested in clustering tasks on two publicly available multi-view data sets to show the superior performance.
\end{abstract}
\section{Introduction}

Empirically, the pairwise constraint is a kind of economic side-information that can be  collected easily from the user feedback. For a user or annotator, it is more convenient to judge whether two images should be in the same category than to classify or tag them. 

During the past decades, a series of important progresses in the utilization of the pairwise constraint have been proposed. The pairwise constraint is widely used as the side-information in the metric learning \cite{xing2002distance,kulis2012metric}, which shows a remarkable improvement in the performance. These methods only pay close attention to the data instances with constraints and the others  with no constraints will not be adjusted by the information. Whereas, in \cite{lu2008constrained} an approach of affinity propagation by Gaussian process is proposed and \citeauthor{eaton2010multi} proposed an approach of the constraint propagation based on the constrained k-Means \shortcite{eaton2010multi}. Using the idea of label diffusion \cite{zhou2004learning} for reference, some constraint propagation algorithms based on the diffusion have been proposed, such as Exhaustive and Efficient Constraint Propagation (E$^2$CP) \cite{lu2010constrained} and Symmetric Graph Regularized Constraint Propagation (SRCP) \cite{fu2011symmetric}. Subsequently, \citeauthor{fu2011multi} proposed the method Multi-Modal Constraint Propagation (MMCP) to extend  E$^2$CP to a multi-view situation  \shortcite{fu2011multi}.

In unsupervised or semi-supervised multi-view problems, a difficulty is there is no enough training data to learn the weights of different views. 
Since the weight reflects the importance of each view, some previous works proposed that the weights should be small if we have some prior knowledge that some views are noisy \cite{kumar2011co,liu2013multi}.
In MMCP , the prior probability of each graph is manually set, which means we need to decide the importance of each view by hand. 
The manual setting can be more difficult if there are too many views.
Some other previous works, like \cite{wang2009unified,xu2016discriminatively}, regard the weights as  variables in their objective function, which combine different views at a view level and can only be solved by iterative update.

Instinctively, it is more reasonable to learn the weights from the robustness of views. 
Some progresses have been made to create a robust affinity matrix for spectral clustering \cite{pavan2007dominant,premachandran2013consensus,zhu2014constructing}. These methods can  be regard as a criteria to judge whether a view is noisy, which makes it possible to have some prior knowledge of the views. 

 Therefore, in this paper, we propose a novel method called  Consensus Prior Constraint Propagation (CPCP), which learns the unified affinity with constraint propagation from the consensus information at a data instance level. 
 Different from the proposals expressed in earlier works, our work is established on the robustness of the neighborhood of data instances, rather than the robustness of each view. We focus our attention on the probabilities contained in the multi-view constraint propagation. We adopt Consensus k-NN \cite{premachandran2013consensus} to produce the conditional probability of each view given data instance. Consequently, we derive the unified transition probability and   affinity matrix from this probability. Moreover, we also introduce a method to balance the positive and negative  constraints based on the objective function of MMCP and provide a straightforward way to make use of the result of constraint propagation.

 The main contributions of our work is: 1) We propose a novel method to derive the importance of each view at a instance level from the consensus information. 2) We introduce a straightforward way to mitigate the imbalance between the must-link and the cannot-link in the constraint propagation. 3) In our approach, we can generate a unified affinity matrix as a result, rather than adjusting the affinity matrix from each view with the result of propagation with .

\section{Consensus Prior Constraint Propagation}
    In this subsection, we first briefly overview the framework of MMCP and the Consensus k-NN algorithm, followed by the details of our Consensus Prior Constraint Propagation.

\subsection{Multi-Modal Constraint Propagation}

Firstly, given a data set with a set of data instances $\mathcal{U} = \{u_1,\dots,u_n \}$, each data instance is denoted by $u_i$. Assume that there are $S$ different views in our multi-view data set, hence we can have $S$ graphs $G_s = (\mathcal{U},W_s)$, $s=1,\dots,S$, where $ W_s$ is the corresponding affinity matrix of view $s$ built by k-NN neighborhood with Gaussian kernel \cite{belkin2001laplacian}. The set $\mathcal{M} = \{(u_i,u_j)\}$ denotes the positive pairwise constraints that $u_i \text{ and } u_j$ should be in the same class, namely must-link. $\mathcal{C} = \{(u_i,u_j)\}$ denotes the set of negative pairwise constraints, namely cannot-link.

Taking the pairwise constraints into consideration, we also build a side-information matrix $Y$, 
\begin{equation}
	Y_{i,j} = 
    \begin{cases}
        1, \qquad&(u_i,u_j)\in \mathcal{M};\\
        -1, &(u_i,u_j)\in \mathcal{C};\\
        0, &\text{otherwise}.
    \end{cases}
    \label{eqY}
\end{equation}

With respect to each graph $G_s$, we build the diagonal matrix $D_s$ from $W_s$, with $D_{i,i,s} = \sum_j W_{i,j,s}$, as in \cite{chung1997spectral}. Then the transition probability on $G_s$ is 
    \begin{equation}
        P(u_i\rightarrow u_j|G_s) = P(u_j|u_i,G_s) = \frac{W_{i,j,s}}{D_{i,i,s}}
        \label{eqmmcp_trans}
    \end{equation}
and the probability of $u_i$ on $G_s$ is 
    \begin{equation}
        P(u_i|G_s) = \frac{D_{i,i,s}}{\sum_iD_{i,i,s}}.
        \label{eqPulg}
    \end{equation}

With $P(G_s)$, the manually decided prior probability of graph $G_s$, we can compute the unified transition probability $P(u_i\rightarrow u_j)$ and the 
probability of $u_i$, namely $P(u_i)$.
%
%
%

Define a matrix $\hat{L}$: 
\begin{equation}
	\hat{L} = \Pi - \frac{\Pi P+P^T\Pi}{2}
	\label{}
\end{equation}
, where $\Pi$ is a diagonal probability matrix with $i$-th diagonal element $P(u_i)$ and $P$ is the unified transition matrix with element $P(u_j|u_i)$.

If the result of the vertical pairwise constraint propagation is denoted by $F_v$. The optimization problem of vertical propagation is 
    \begin{equation}
        \mathop{\mathrm{min}}_{F_v}\;\frac{1}{2}\eta \mathrm{tr}((F_v-Y)^T\Pi (F_v-Y))+\frac{1}{2}\mathrm{tr}(F_v^T\hat{L}F_v).
        \label{eqmmcp}
    \end{equation}
Solving this optimization problem by differentiating, we can find the closed-form result of vertical propagation,
    \begin{equation}
        F_v = \eta(\eta \Pi+\hat{L})^{-1}\Pi Y.
        \label{}
    \end{equation}
The result of the horizontal propagation is similar. By combining the results of vertical and horizontal propagation, we attain the final result of the constraint propagation
    \begin{equation}
        F = \eta^2(\eta\Pi+\hat{L})^{-1}\Pi Y\Pi(\eta\Pi+\hat{L})^{-1}.
        \label{}
    \end{equation}

\subsection{Consensus k-NN}

The Consensus k-NN algorithm proposed in \cite{premachandran2013consensus} collects the consensus information of multiple rounds of k-NN neighborhood to provide a criteria for the neighborhood selection. If a pair of nodes $u_i$ and $u_j$ keeps appearing among the k-NN neighborhood of many of other nodes, the chance of these two nodes being similar is much higher. In contrast, if the distance between $u_i$ and $u_j$ is quite short but they never appear among the k-NN neighborhood of other nodes, it is more likely to be the  noise.
In Algorithm \ref{al_consknn} we show some details of the consensus matrix $C$ in Consensus k-NN. And there is a threshold $\tau$ in  Consensus k-NN. If the $C_{i,j}>\tau$, $u_i \text{ and }u_j$ will be contained in the neighborhood set of each other.

    \begin{algorithm}[t]
        \renewcommand{\algorithmicrequire}{\textbf{Input:}}
        \renewcommand{\algorithmicensure}{\textbf{Output:}}
        \caption{Consensus matrix of Consensus k-NNs}
        \label{al_consknn}
        \begin{algorithmic}[1]
           \STATE $C = 0$;
           \FOR {$i = 1:N$}
					 \FORALL {$u_j, u_k \text{ such that } u_j, u_k \in \text{ k-NN}(u_i)$}
           \STATE $C_{j,k} = C_{j,k}+1$;
           \STATE $C_{k,j} = C_{k,j}+1$;
           \ENDFOR
           \ENDFOR

        \end{algorithmic}
    \end{algorithm}

\subsection{Consensus Prior Knowledge}

The consensus k-NN used in our approach is a little different from the algorithm proposed in \cite{premachandran2013consensus}. We implement consensus k-NN to prune noise edges from the existing affinity matrix, rather than build new neighborhood sets. We set
    \begin{equation}
        W^{cons}_{i,j} = 
        \begin{cases}
            0,\quad\qquad\qquad&\text{if }C_{i,j}<\tau;\\ W^{dense}_{i,j},&\text{otherwise.}
        \end{cases}
        \label{eqcons}
    \end{equation}
    , where $W^{dense}$ is a dense affinity matrix build from k-NN neighborhood and $W^{cons}$ is the consensus k-NN affinity matrix. Since the affinity matrix we build from the k-NN neighborhood is relatively dense, which we will elaborate in following paragraphs, the algorithm we implement in our approach can generate an approximate result  and be more efficient.

    In CPCP, the conditional probability of each graph $P(G_s|u_i)$ is generated from the consensus information and most of the other probabilities are derived from such information. Since the prior knowledge of the views are derived from the consensus, we name it consensus prior knowledge. As the diagram Fig. \ref{fig:diag1} shows, the unified transition probability of multiple views can be attained once we obtain the pivotal probability $P(G_s|u_i)$.

    \begin{figure}[t]
        \centering
        \includegraphics[width = 0.8\columnwidth]{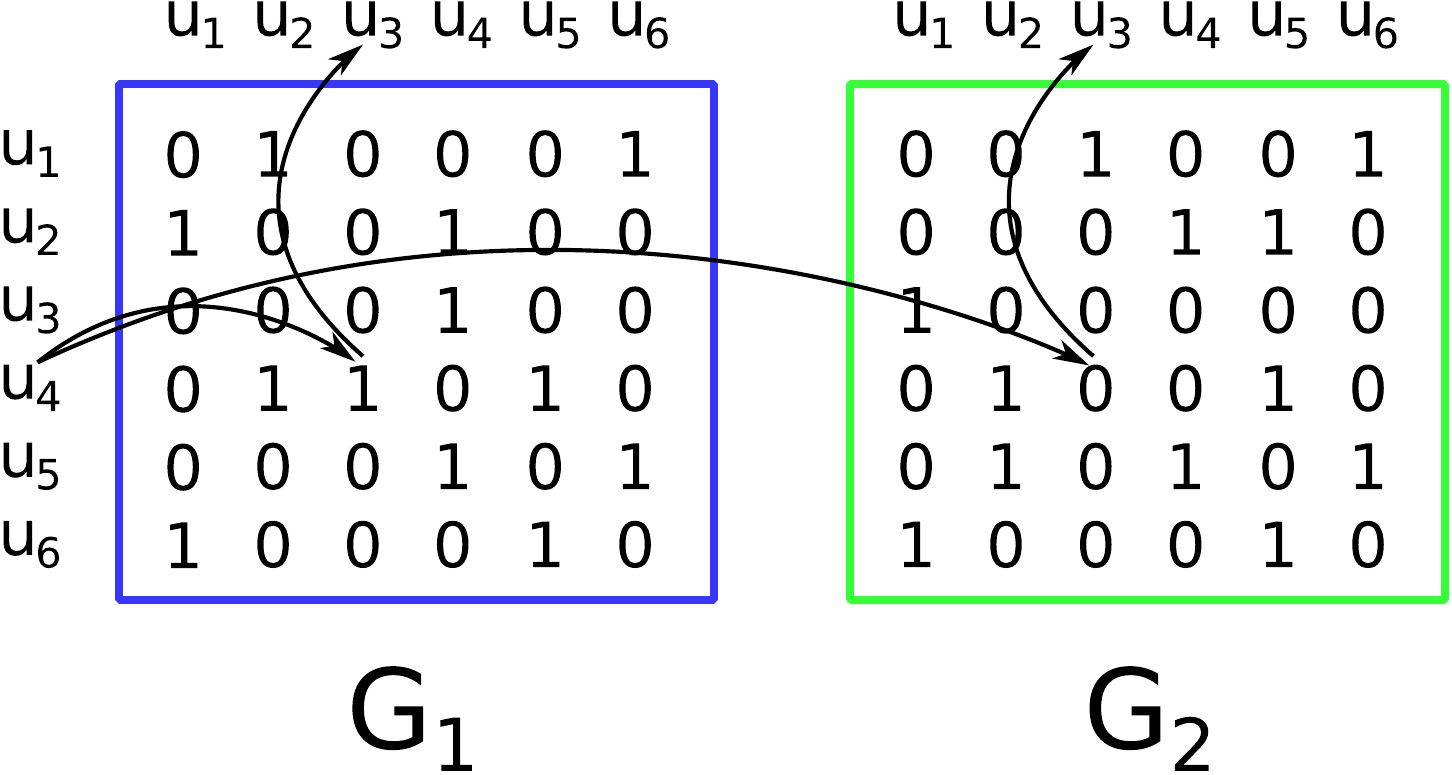}
				\caption{A diagram of transition probability of multiple views. The horizontal arrows indicate the conditional probability of each graph and vertical arrows indicate the internal transition probability of each view. In this example we have $P(u_4\rightarrow u_3) = P(u_3|u_4,G_1)P(G_1|u_4)+P(u_3|u_4,G_2)P(G_2|u_4)$.}
        \label{fig:diag1}
    \end{figure}

    With respect to each graph $G_s$, we construct the consensus k-NN affinity matrix $W^{cons}_s$ according to Eq. \ref{eqcons}. Subsequently we can define the consistency of each data instance in different views by the KL-divergence:
\begin{equation}
	c_s(i) = \sum_j W^{dense}_{i,j,s}\;\text{ln}\frac{W^{dense}_{i,j,s}}{W^{cons}_{i,j,s}}
    \label{eqconsistency}
\end{equation}
, where the $W^{dense}_{i,j,s}$ is the k-NN neighborhood affinity between data $u_i$ and $u_j$ in the graph $G_s$ and $W^{cons}_{i,j,s}$ is the corresponding consensus k-NN affinity. Both $W_s^{dense}$ and $W_s^{cons}$ are normalized with row summation one, to make each row as probability distribution. Here the consistency $c_s(i)$ is a metric of the robustness of the relationship between data instance $u_i$ and its neighborhood in graph $G_s$. The smaller is the consistency, the more stable is the relationship. Because of the sparsity of the affinity matrices, two essential tricks are performed in our experimental. One is that the k-NN neighborhood affinity $W^{dense}_s$ are relatively dense matrices with $k = \#sample/\#cluster$. This setting is also necessary in the following probabilities. The other one is that we add a minute quantity (such as $10^{-8}$) to each element of $W_s^{dense}$ and $W_s^{cons}$ to avoid overwhelming zeros in the discrete probability.

Once we have the consistency, we can obtain the conditional probability of graph $G_s$ given $u_i$ as follows.
\begin{equation}
    P^\dagger (G_s|u_i) = \frac{(c_s(i)+1)^{-1}}{\sum_i (c_s(i)+1)^{-1}}
    \label{eqPglu}
\end{equation}
From Eq. \ref{eqPglu} we can notice that with a given view $s$ a smaller consistency implicates a greater conditional probability of graph $G_s$. The conditional probability $P^\dagger(u_i|G_s)$ is calculated in the same way as Eq. \ref{eqPulg}. 

It needs to pay attention here that the superscript $\dagger$ of $P^\dagger (G_s|u_i)$ and $P^\dagger (u_i|G_s)$ means the conditional probabilities are not proper. They are pseudo-conditional probability distributions. It is because that there may not exist $P(u_i), P(G_s)$ and $P(u_i,G_s)$ which can generate such two conditional distributions. $P^\dagger(G|u)$ and $P^\dagger(u|G)$ do not satisfy such causality constraint of the pair of conditional distribution and it makes these two distributions illegal. Under the condition that there is no zero in these discrete conditional probability distributions, we have the following proposition.
\begin{prop}
    Define a quotient matrix $Q$, where $Q_{i,s} = \frac{P^\dagger(u_i|G_s)}{P^\dagger(G_s|u_i)}$. $P^\dagger(G|u)$ and $P^\dagger(u|G)$ are legal conditional probability distributions if and only if the quotient matrix $Q$ has rank 1 and $\sum_s\frac{1}{\sum_i Q_{i,s}} = 1$.
		\label{prop1}
\end{prop}
\begin{proof}
    $(\Rightarrow)$:\quad Since $P^\dagger(G|u)$ and $P^\dagger(u|G)$ are legal conditional probability distributions, we can always find the corresponding probability distribution $P(u), P(G)$ and $P(u, G)$. Therefore we have 
    \begin{equation}
        Q_{i,s} = \frac{P^\dagger(u_i|G_s)}{P^\dagger(G_s|u_i)} = \frac{P(u_i)}{P(G_s)}.
        \label{eqq}
    \end{equation}
    Given row $i$ of the matrix $Q$, row $j$ can be determined by multiplying $\frac{P(u_i)}{P(u_j)}$ by row $i$. It leads to a rank 1 matrix $Q$. 
    Under the condition Eq. \ref{eqq}, we also have
    \begin{equation}
        \sum_s\frac{1}{\sum_i Q_{i,s}} = \sum_s\frac{1}{\sum_i \frac{P(u_i)}{P(G_s)}} = \sum_s P(G_s) = 1    
        \label{}
    \end{equation}

    $(\Leftarrow)$:\quad Assume $Q$ is a rank 1 matrix, $Q$ can be factorized as a product of two positive vectors, which gives $Q = mn^T$. From the factorization we can derive the unique $m$ and $n$, if we set $\sum_i m_i = 1$. By substitution of $Q_{i,s} = m_in_s$ into the equation $\sum_s\frac{1}{\sum_i Q_{i,s}} = 1$, it leads to the result $\sum_s\frac{1}{n_s} = 1$.
Considering the property of $m$ and $n$, we can interpret these two vectors as two marginal probability distributions:
    \begin{equation}
        P(u_i) = m_i, \quad P(G_s) = n_s
        \label{}
    \end{equation}

    Now we need to prove there exists a joint probability distribution $P(u, G)$ that satisfies $P(u, G) = P^\dagger(G|u)P(u) = P^\dagger(u|G)P(G)$. 
		$P^\dagger(G|u)$ and $P(u)$ are two probability distributions, thus the product of them is a probability distribution. Similarly $P^\dagger(u|G)P(G)$ is also a probability distribution. Hence our goal is to prove that $P^\dagger(G_s|u_i)P(u_i) = P^\dagger(u_i|G_s)P(G_s)$ holds for all $i$ and $s$.
    Making use of $Q_{i,s} = m_in_s$, we then obtain that 
    \begin{equation}
        \begin{split}
            &\frac{P^\dagger(u_i|G_s)}{P^\dagger(G_s|u_i)} = Q_{i,s} = m_in_s\\
            \Rightarrow &\frac{P^\dagger(u_i|G_s)}{n_i} = P^\dagger(G_s|u_i)m_s\\
            \Rightarrow &P^\dagger(u_i|G_s)P(G_s) = P^\dagger(G_s|u_i)P(u_i)
        \end{split}
        \label{}
    \end{equation}
    , which shows that $P^\dagger(G|u)$ and $P^\dagger(u|G)$ are a pair of legal conditional probability distributions.

\end{proof}

With Proposition \ref{prop1}, we notice that if we want the $P^\dagger(G|u)$ and $P^\dagger(u|G)$ to be the legal conditional distributions and derive $P(u_i), P(G_s)$ from them. We need the quotient matrix $Q$ to have rank 1. Here we ignore the condition  $\sum_s\frac{1}{\sum_i Q_{i,s}} = 1$, which can be satisfied by scaling. Generally, $Q$ will not  be a rank 1 matrix. However, if the $P^\dagger(G|u)$ and $P^\dagger(u|G)$ are approximate to the true value, $Q$ will be similar to a rank 1 matrix. Particularly, we apply singular value decomposition to the matrix $Q$ to obtain a approximation $\hat{Q}$ with rank 1. As a result of new quotient matrix $\hat{Q}$, only the $P(u_i), P(G_s)$ can be derived from $\hat{Q}$ directly and we rebuild $P(G|u)$,  $P(u|G)$ and the corresponding $P(u, G)$ with the constraint  $\hat{Q}$. With following optimization problem, we aim to solve proper $P(G|u)$ and $P(u|G)$ as similar to $P^\dagger(G|u)$ and $P^\dagger(u|G)$ as possible.

    \begin{equation}
        \begin{split}
					\mathop{\mathrm{min}}_{P(u_i|G_j), P(G_j|u_i)} & \quad \frac{1}{2} \alpha \sum_{i,j}(P(u_i|G_j)-P^\dagger(u_i|G_j))^2 \\
					&+ \frac{1}{2} \beta \sum_{i,j}(P(G_j|u_i)-P^\dagger(G_j|u_i))^2\\
						\mathrm{s.t.}\quad \quad \quad&\quad \frac{P(u_i|G_j)}{P(G_j|u_i)} = \hat{Q}_{i,j}\\
        \end{split}
        \label{eqOpt}
    \end{equation}

		We remove the normalization constraints  $\sum_{i}P(u_i|G_j) = 1$ and $\sum_{j}P(G_j|u_i) = 1$ in Eq.\ref{eqOpt}. With such a relaxation, this optimization problem can be solved element by element efficiently, and the normalization process will be finished after the optimization. In order to balance two part of the objective function, we set $\alpha = \frac{1}{\sum_{i,j}P(u_i|G_j)^2}$ and $\beta = \frac{1}{\sum_{i,j}P(G_j|u_i)^2}$.
		By solving this optimization problem, we can obtain a closed-form solution.


		Finally we obtain the unified affinity matrix with Eq. \ref{eqmmcp_trans} by 
		\begin{equation}
			W_{i,j} = P(u_i,u_j) = P(u_i)\sum_s P(u_j|u_i, G_s)P(G_s|u_i)
		\end{equation}
        , and the matrix is sparsified as a k-NN neighborhood.

		\subsection{Balance Between Must-link and Cannot-link}
		
    Here, we separate constraint matrix $Y$ into two parts $Y = Y_++Y_-$, $Y_+$ for the positive elements and $Y_-$ for the negative ones. Then we have
    \begin{equation}
        \begin{split}
             &\mathrm{tr}((F_v - Y )^T\Pi(F_v - Y))\\
            =\; & \mathrm{tr}(F_v^T\Pi F_v+Y_+^T\Pi Y_++Y_-^T\Pi Y_- -2F_v^T\Pi Y_+ \\
            &- 2F_v^T\Pi Y_- + 2Y_+^T\Pi Y_- )\\
        \end{split}
        \label{fpif}
    \end{equation}
		Since in Eq. \ref{fpif} we have $Y_+^T\Pi Y_- = \bf{0}$, we can obtain that Eq. \ref{fpif} equals to
    \begin{equation}
        \begin{split}
						\;&\mathrm{tr}((F_v-Y_+)^T\Pi(F_v-Y_+))\\
            &+\mathrm{tr}((F_v-Y_-)^T\Pi(F_v-Y_-))-\mathrm{tr}(F_v^T\Pi F_v)
        \end{split}
        \label{fpif2}
    \end{equation}
		
		As we separate  constraints into two parts, we can weight the positive part with a parameter based on ratio of must-link to cannot-link $\alpha = \sqrt{\#negative/\#positive}$. Our objective function is obtained by substituting the weighted constraint into Eq. \ref{eqmmcp} to give
    \begin{equation}
        \begin{split}
            \mathop{\mathrm{min}}_{F_v}\quad&\frac{1}{2}\alpha\eta \mathrm{tr}((F_v-Y_+)^T\Pi(F_v-Y_+))\\
            &+\frac{1}{2}\eta \mathrm{tr}((F_v-Y_-)^T\Pi(F_v-Y_-))\\
            &+\frac{1}{2}\mathrm{tr}(F_v^T(L-\eta \Pi)F_v)
        \end{split}
        \label{obj2}
    \end{equation}
		, where the $L$ is different from the $\hat{L}$ in Eq. \ref{eqmmcp} and $L = \Pi-W$.
		By differentiating the function in Eq.\ref{obj2} with respect to $F_v$ and setting it to zero as \cite{fu2011multi} did, we can obtain the matrix $F_v$ after vertical propagation on multiple graphs
    \begin{equation}
        \begin{split}
            & F_v = \eta(L+\alpha\eta\Pi)^{-1}\Pi(\alpha Y_++Y_-)
        \end{split}
        \label{}
    \end{equation}
		
		By combining the results of vertical and horizontal propagation, the final result of the constraint propagation is 
		\begin{equation}
			\begin{split}
				& F = \eta^2(L+\alpha\eta\Pi)^{-1}\Pi(\alpha Y_++Y_-)\Pi(L+\alpha\eta\Pi)^{-1}
			\end{split}
			\label{}
		\end{equation}

        \subsection{Affinity with Constraint Propagation}
        Most of the previous works employ the constraint propagation result $F$ to refine the k-NN neighborhood affinity matrix. The adjustment of the affinity matrix leads to a sparse matrix as the result of their approaches. Different from them, there is no adjustment in our approach. We instinctively generate the final affinity matrix from the $F$ itself. The affinity with constraint propagation $W^*$ is build with the Sigmoid activation of $F$, 
        \begin{equation}
            W^*_{i,j} = 
            \begin{cases}
                \frac{1}{1+\text{exp}(-F_{i,j}/\sigma)} \qquad &\text{if }F_{i,j}>0;\\
                0 &\text{otherwise}
            \end{cases}
            \label{}
        \end{equation}
				, where $\sigma$ is the average magnitude of the elements in $F$. The $W^*$ is not a sparse matrix as the other methods, the number of zeros in $W^*$ is related to the  ratio of the positive constraints to the negative ones. Finally, we employ the spectral clustering \cite{von2007tutorial} on the affinity $W^*$ to form the clusters.

\begin{table}[t]
	\caption{Description of data sets}
	\label{tab_data}
	\centering
		\begin{tabular}{c l l}
			\hline
			View & Corel 5k & PASCAL VOC'07\\
			\hline
            1& Lab (4096) & Lab (4096) \\
            2& DenseSift (1000) & DenseSift (1000) \\
            3& annot (260)& tags (804) \\
            4& Hsv (4096) & Hsv (4096) \\
            5& Gist (512) & Gist (512) \\
            6& RgbV3H1 (5184) & RgbV3H1 (5184) \\
            7& HarrisHueV3H1 (300) & HarrisHueV3H1 (300) \\
            8& HsvV3H1 (5184) & HsvV3H1 (5184) \\
            \hline
            Images &4999 & 9963 \\
            Classes &50 & 20 \\
			\hline
		\end{tabular}
\end{table}

\section{Experimental Results}

In this section, we conduct several experiments to demonstrate the performance of the proposed approach CPCP on two benchmark data sets. We compare our CPCP algorithm with some state-of-the-art methods. The clustering result of Multi-Modal Constraint Propagation (MMCP) \cite{fu2011multi} and the method for single source data Exhaustive and Efficient Constraint Propagation (E$^2$CP) \cite{lu2010constrained} are also reported in the this section. Moreover, we use the Normalized Cuts \cite{shi2000normalized} with no pairwise constraints as the baseline method in the evaluation of clustering performance. There are some more recent multi-view constraint propagation methods Unified Constraint Propagation (UCP) \cite{lu2013unified} and Multi-Source Constraint Propagation (MSCP) \cite{lu2013exhaustive}) which also have satisfactory performance, but these methods are designed specifically for the case of 2 views and are difficult to extend to more views. Hence they are not included in our experiments.

\subsection{Data Sets}
We consider two benchmark image data sets with the textual description, which have been used in previous work.

{\bf Corel 5k}. This data set is an important benchmark which has been widely used in many tasks. It contains 50 classes with approximate 5000 image. Each image is annotated with several keywords from a dictionary of 260 words.

{\bf PASCAL VOC'07}. This data set \cite{pascal-voc-2007} contains 20 different object categories and around 10000 images. All of the images of PASCAL VOC'07 are annotated with one or more categories.   

For the ease of experimental results reproduction or direct comparison, we employ the publicly available features, INRIA features \cite{guillaumin2009tagprop}, of these two data sets instead of extracting features by ourselves. 
Tab. \ref{tab_data} summarizes a part of the features that we used for the clustering performance evaluation in our experiments.

\begin{figure}[t]
    \centering
    \subfigure[Corel 5k]{
        \includegraphics[width = 0.47\columnwidth]{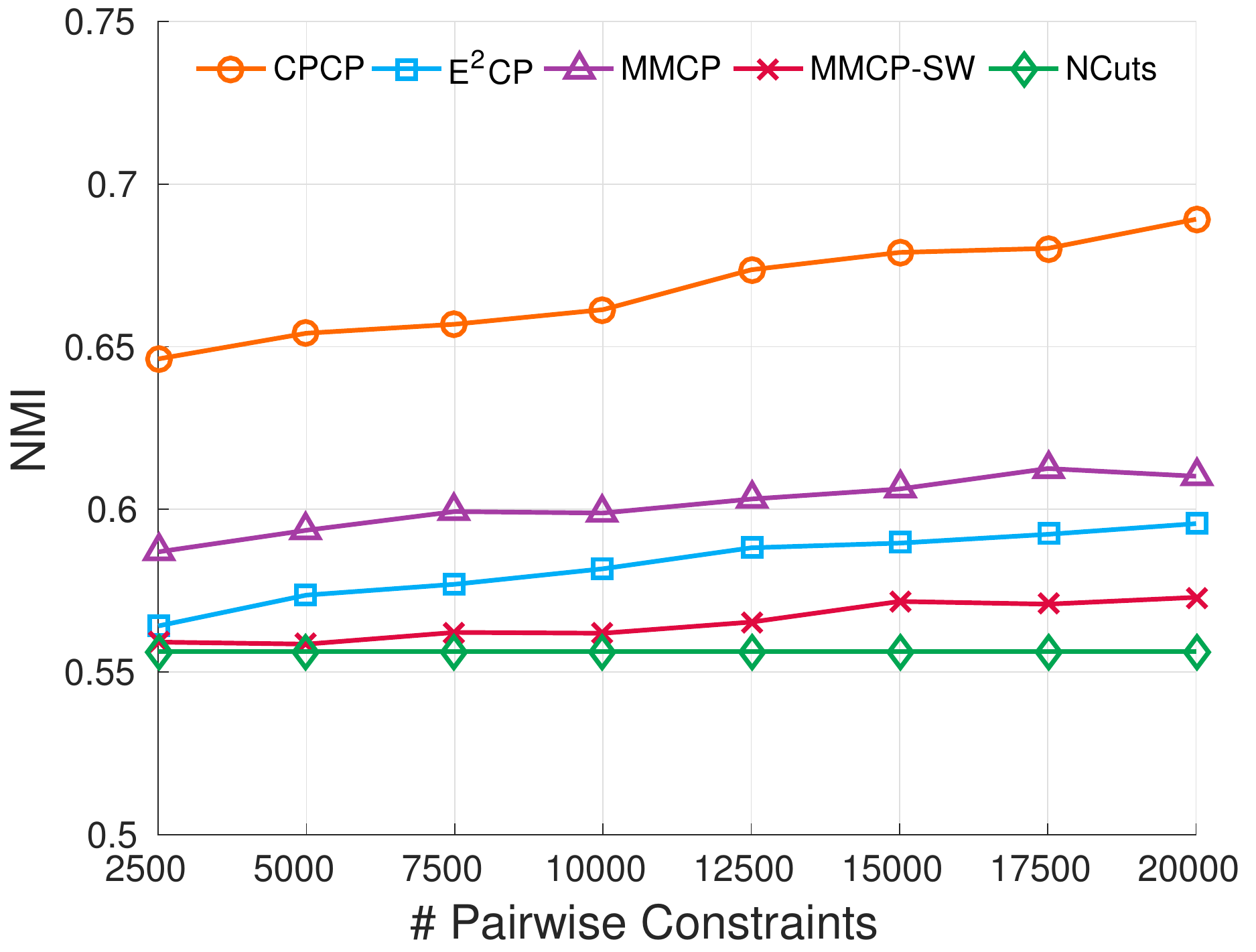}
        \label{fig:3viewnmi1}
    }
    \subfigure[PASCAL VOC'07]{
        \includegraphics[width = 0.47\columnwidth]{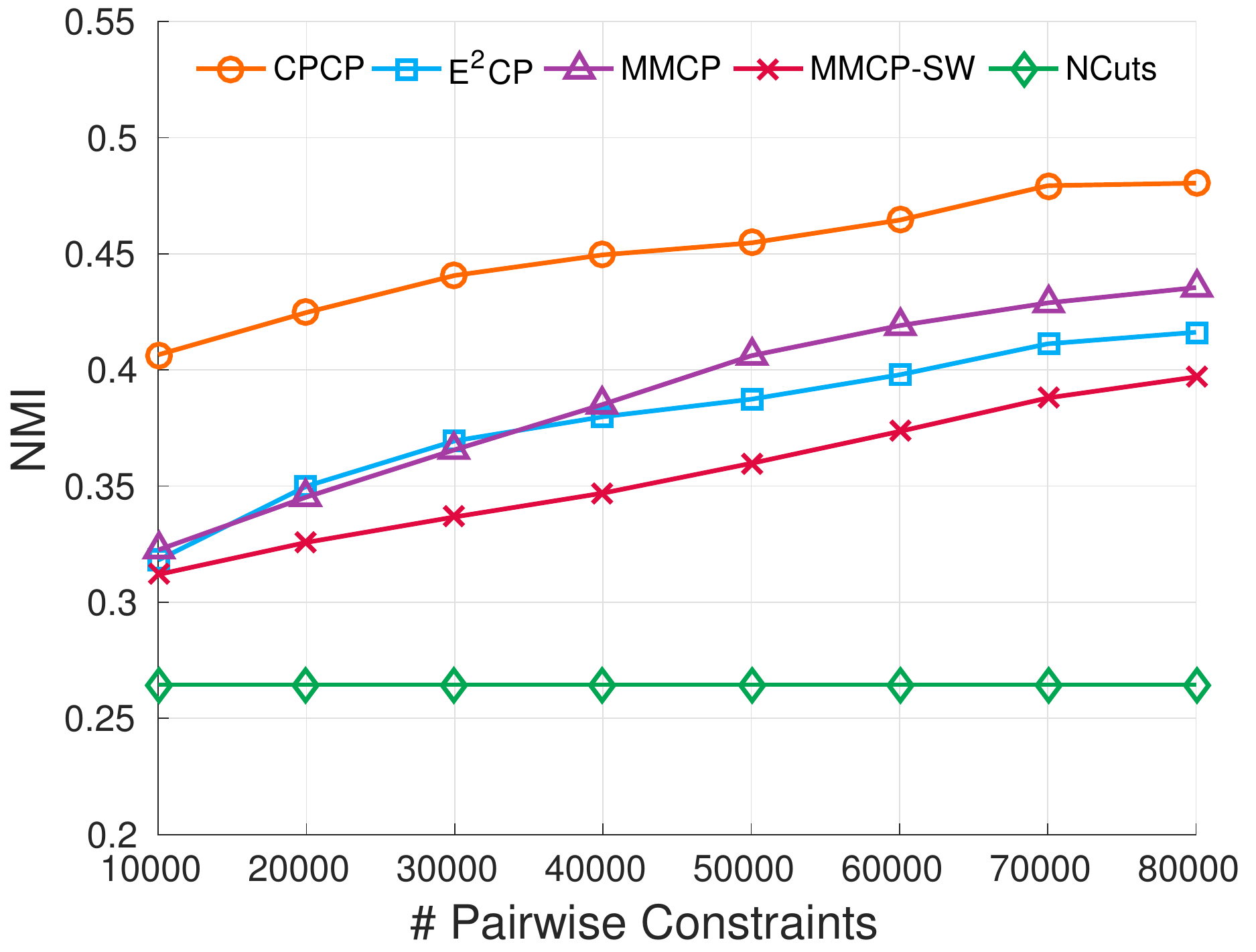}
        \label{fig:3viewnmi2}
    }
    \caption{Clustering result on three views of Corel 5k and PASCAL VOC'07 with different number of constraints.}
    \label{fig:3viewnmi}
\end{figure}

\begin{figure}[t]
    \centering
    \subfigure[Corel 5k]{
        \includegraphics[width = 0.47\columnwidth]{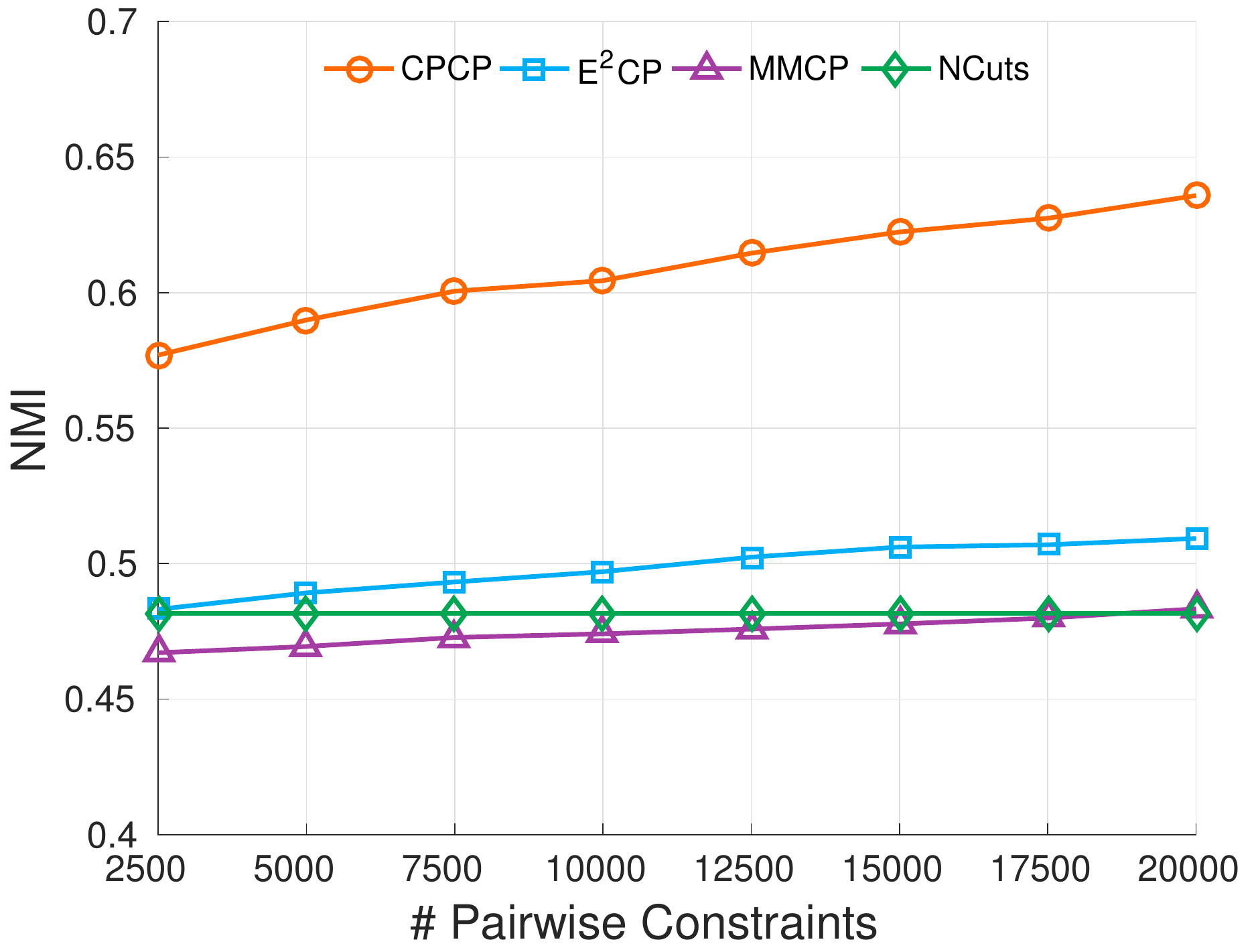}
        \label{fig:8viewnmi1}
    }
    \subfigure[PASCAL VOC'07]{
        \includegraphics[width = 0.47\columnwidth]{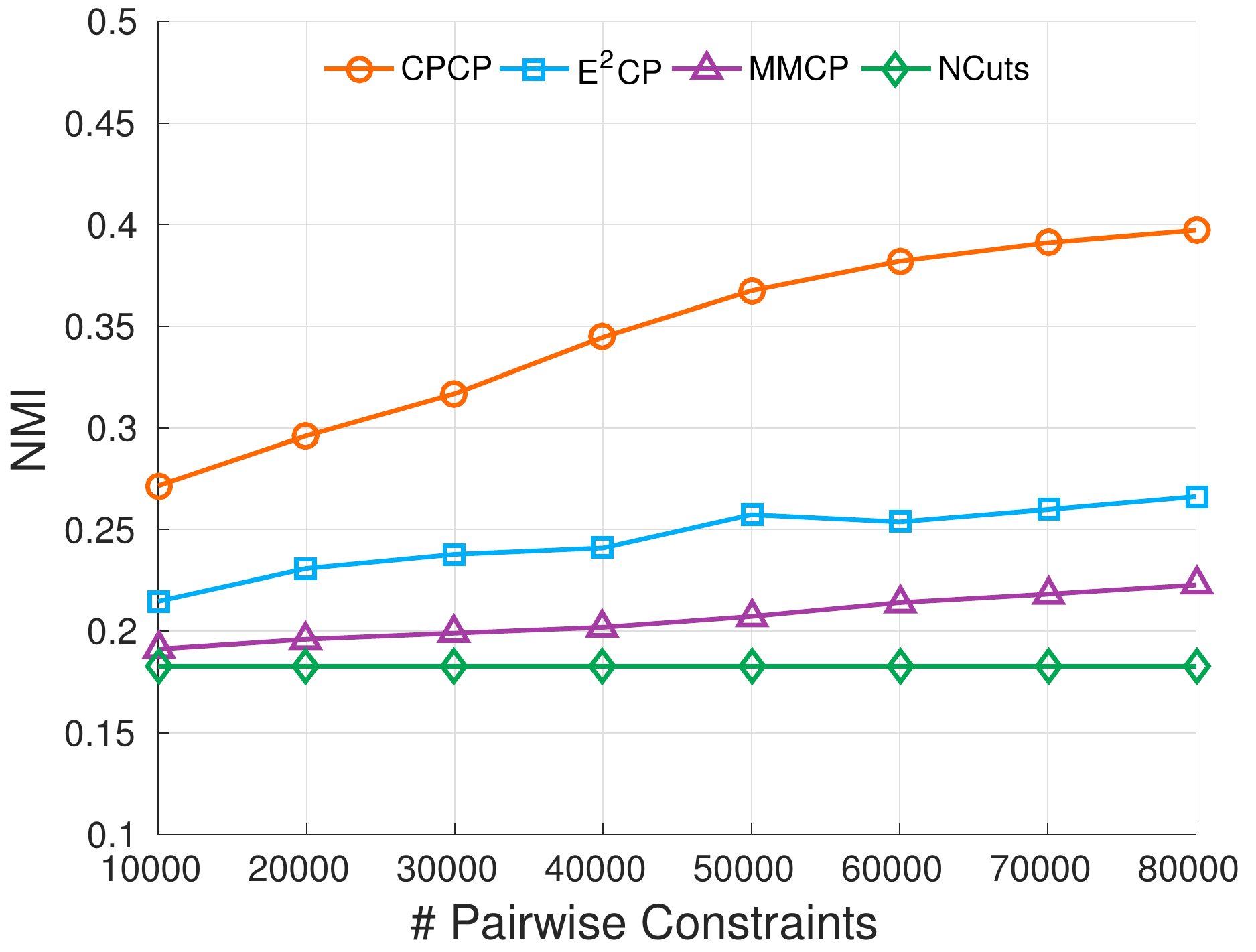}
        \label{fig:8viewnmi2}
    }
    \caption{Clustering result on eight views of Corel 5k and PASCAL VOC'07 with different number of constraints.}
    \label{fig:8viewnmi}
\end{figure}

\subsection{Experiment Setup}

In the experiments of clustering evaluation, we use one textual feature and seven other image features of INRIA features. All the image features are normalized with zero mean and scaled to $[-1,1]$. The affinity matrix of image features are computed by Gaussian kernel and the affinity of textual feature is produced from $cosine$ distance. It is worth noting that about one-third of images in  PASCAL VOC'07 don't have any tag. Thus, there will be around 3000 zero rows in the affinity matrix of this view. To deal with such blank feature, we add some random noise, which is minute and can be ignored, to every dimension of these data instance. 

\begin{table*}[t]
    \addtolength{\tabcolsep}{3pt}
	\caption{Clustering performance (NMI) on of E$^2$CP and MMCP on Corel 5k and PASCAL VOC'07}
	\label{tab_e2cp}
	\centering
		\begin{tabular}{|l| c| c| c| c| c| c| }
			\hline
            \multirow{2}{100pt}{\centering Method} &\multicolumn{3}{ |c }{Corel 5k} &\multicolumn{3}{ |c |}{PASCAL VOC'07}
			\\
			\cline{2-7}
			& Avg & Max & Min & Avg & Max & Min \\
			\hline
			E$^2$CP-Lab           & 0.3510     &   0.3573   &   0.3460    & 0.0929     & 0.0966     & 0.0860      \\
			E$^2$CP-DenseSift     & 0.2701     &   0.2742   &   0.2661    & 0.2237     & 0.2267     & 0.2212      \\
            E$^2$CP-annot/tags    & 0.6153     &   0.6259   &   0.6051    &{\it 0.4533}&{\it 0.4650}&{\it 0.4364} \\
			E$^2$CP-Hsv           & 0.3409     &   0.3459   &   0.3321    & 0.0790     & 0.0811     & 0.0750      \\
			E$^2$CP-Gist          & 0.2150     &   0.2198   &   0.2115    & 0.1622     & 0.1726     & 0.1582      \\
			E$^2$CP-RgbV3H1       & 0.3111     &   0.3162   &   0.3054    & 0.1037     & 0.1060     & 0.1009      \\
			E$^2$CP-HarrisHueV3H1 & 0.2608     &   0.2652   &   0.2513    & 0.0992     & 0.0966     & 0.1008      \\
			E$^2$CP-HsvV3H1       & 0.3415     &   0.3491   &   0.3351    & 0.0910     & 0.0947     & 0.0860      \\
			\hline
			CPCP-3Views           &{\bf 0.6892}&{\bf 0.6954}&{\bf 0.6845}&{\bf 0.4804}&{\bf 0.4823}&{\bf 0.4725}\\
            CPCP-8Views           &{\it 0.6358}&{\it 0.6400}&{\it 0.6301}& 0.3973     & 0.4009     & 0.3925	  \\	
			\hline
		\end{tabular}
\end{table*}

In order to evaluate the performance of clustering of these method, we adopt the Normalized Mutual Information (NMI) as the measure to compare the clustering result with the given ground-truth. As we mentioned above, PASCAL VOC'07 is a multi-label data set, in which many images has more than one categories in the ground-truth. The multi-label ground-truth makes it difficult to find a matching between the clustering result and the ground-truth.
To deal with such difficulty, we copy the multi-label data instances and separate their labels.  Concretely, assume a data instance $u_i$ with three labels $A, B \text{ and } C$, which can be expressed as a pair $(u_i, \{A, B, C\})$. We separate this data-label pair into three distinct pairs $(u_i, A), (u_i, B) \text{ and }(u_i, C)$ to generate a new single-label ground-truth. Similarly, we make two copies of the corresponding clustering result. For instance, if the clustering result is $(u_i, A)$, we will regard it as three different ones $(u_i, A), (u_i, A) \text{ and } (u_i, A)$ in the new clustering result. In this case, the data instance are seen to be clustered correctly, and the accuracy is 1/3. We can notice that with this trick the accuracy or NMI will be less than 1, even every data sample is clustered correctly. 
We divide the evaluation result by the ideal score to normalize it from 0 to 1. We call this method multi-label augmentation.

In our experiments, we impose a fixed parameter selection criteria, because there is no validation data set in the clustering tasks. We set the size of k-NN neighborhood to be $k = \text{Round}(\text{log}_2(n/c))$, where $n$ is the number of the data instances and $c$ is the number of classes. The propagation parameter $\eta = 0.25$. The embedding dimension in the spectral clustering is $c+1$. 


\subsection{Clustering Performance Evaluation}

In this subsection we will report the performance of  CPCP and the comparison methods on clustering experiments. 

The clustering results of 3 views are shown in Fig. \ref{fig:3viewnmi}. In this experiment, we use the first three views listed in Tab. \ref{tab_data} of two data sets, including two image features and one textual features. We implement two versions of Multi-Modal Constraint Propagation in this 3-views experiment. One is MMCP, which we set the prior graph probabilities of 3 views to 0.2, 0.05 and 0.75 proposed by the authors of MMCP \cite{fu2011multi}. Considering in practice it is difficult to decide the importance of a view manually, the other version is MMCP with same weight (MMCP-SW), which means that we assign each views the same importance and each prior probability is 1/3. Since our CPCP can generate an unified affinity graph after the propagation, the experiments in this section only considering the unified affinity graph. In MMCP, there is a unified graph, into which we can incorporate the propagation result $F$.
As to E$^2$CP, we impose the constraint propagation on each view and fuse each affinity graph after the propagation by linear combination. The number of pairwise constraints used for constrained clustering is 
from 0.01\% to 0.08\% of the total number of pairwise constraints in both data sets. From Fig. \ref{fig:3viewnmi} we can see that our approach CPCP significantly outperforms the other methods on both data sets in this 3-views experiment.

The clustering results of 8 views are shown in Fig. \ref{fig:8viewnmi}. In this experiment we use all the views listed in Table \ref{tab_data}. The details of the 8-views experiment is highly similar to the 3-views case. One difference is that, when we have 8 views, it is almost impossible to decide the importance of each view by hand. Therefore in this experiment, there is only one version of MMCP. Every view in MMCP has the same prior graph probability 1/8. From the figure, we can see that, the task of clustering is becoming harder while the number of views is increasing. Meanwhile, the advantage of our approach is more obvious than the 3-views case.

We also compare our approach CPCP with the constraint propagation on single view. As Tab. \ref{tab_e2cp} shows, we adopt E$^2$CP on eight views  with 20000 pairwise constraints. CPCP-3Views gives the clustering results of CPCP on the first three views, which is also shown in Fig. \ref{fig:3viewnmi}. Similarly CPCP-8Views gives the clustering results of 8-views case. The {\it annot} and the {\it tags} are two different names of the textual feature in Corel 5k and PASCAL VOC'07, thus we write them in the same row. We can see that in both data sets the textual feature has the best performance. 
While the results of the features in PASCAL VOC'07 except textual feature are not satisfactory. Empirically, we have no idea that which view is a satisfactory representation. It makes the results of multi-view clustering worse when we keep increasing the number of views, since some views which do not have the discriminative representations can be regarded as the noise. That is the reason why our CPCP has the best performance when we use 3 views, but when we use 8 views CPCP only has the second best perform in Corel 5k and the third best in PASCAL VOC'07.

\begin{figure}[t]
    \centering
    \includegraphics[width = 0.8\columnwidth]{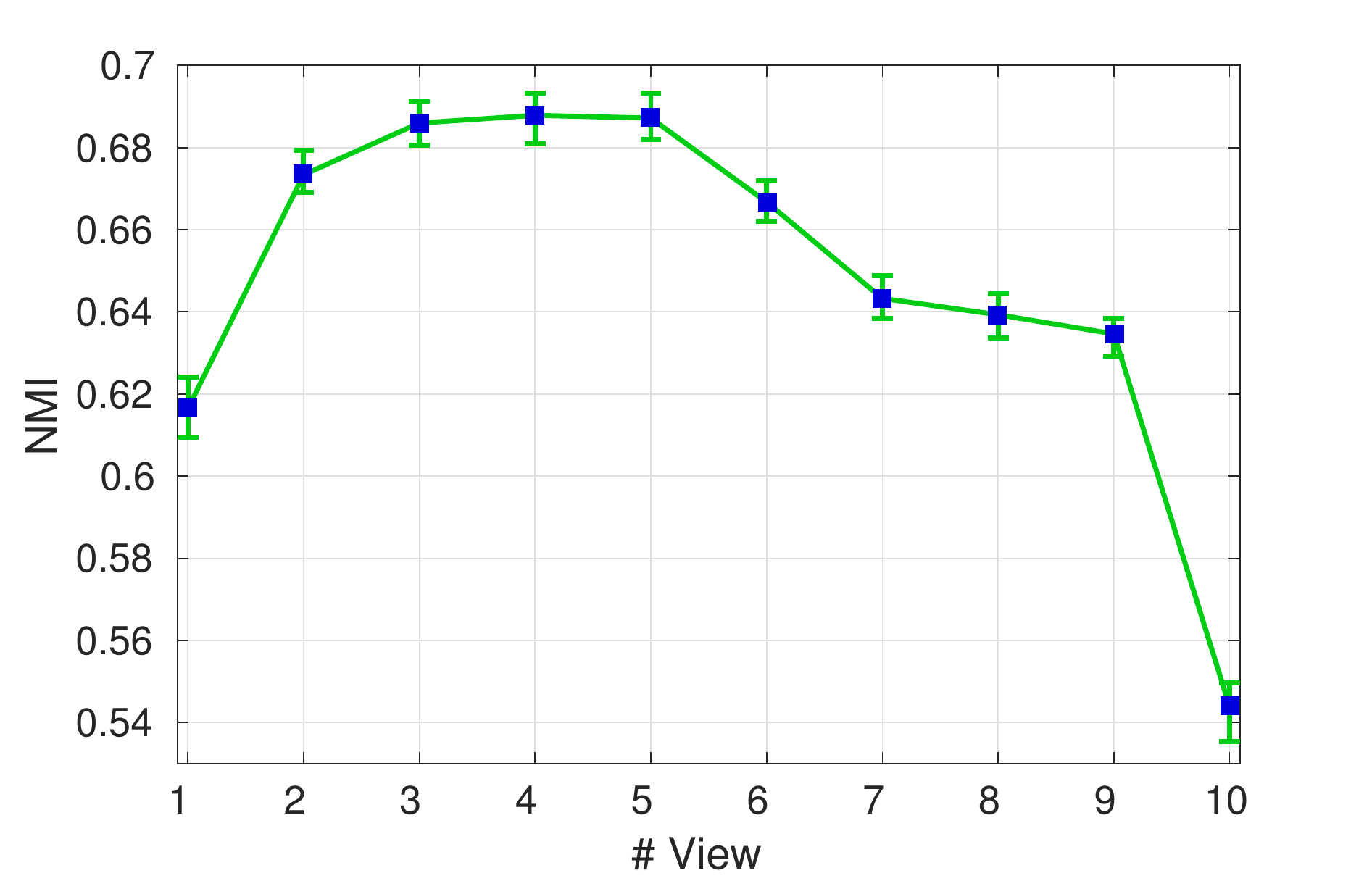}
    \caption{Clustering results on Corel 5k data set with different number of views.}
    \label{fig:corel5k_1_10}
\end{figure}

\subsection{View Selection}

Because of the ability of weighting the importance in the multi-view problem, we can regard the consensus information of CPCP as a guidance in the view selection. With the consensus information, the marginal probability $P(G_s)$ of the each graph can be generated immediately. After sorting the marginal probabilities, we can eliminate the view with the smallest probability. Fig. \ref{fig:corel5k_1_10} shows the clustering results when impose CPCP on different number of views. In this experiment, we selection 10 views as the initialization ({\it annot, Lab, DenseSift, Hsv, Gist, RgbV3H1, HarrisSiftV3H1, HsvV3H1, HarrisSift, DenseSiftV3H1}). By the strategy mentioned above, we eliminate the one view at the time leaving the views which have more contributions to the clustering. As the figure demonstrates, the NMI increases rapidly when we remove the one view at the first time. The result of clustering has a continuous improvement as we remove the views one by one. Here we have the best performance with four views, and if we keep decrease the number of views the NMI will go down. Although there is no criteria that can gives a proper number of views, it is possible to eliminate some worst views which will corrupt the clustering performance by considering  consensus information.

\section{Conclusions}

In this paper, we present a novel multi-view constraint propagation approach, called Consensus Prior Constraint Propagation. In our method, the unified affinity matrix after constraint propagation is produced from consensus information of each data instance, and the imbalance between positive and negative constraints is solved.
Extensive experiments demonstrate the superiority if the proposed method CPCP.

\bibliographystyle{aaai} 
\bibliography{bibfile}
\end{document}